\documentclass[letterpaper, 10 pt, conference]{ieeeconf}  % Comment this line out
\IEEEoverridecommandlockouts                              % This command is only
\usepackage{graphics} % for pdf, bitmapped graphics files
\usepackage{amsmath} % assumes amsmath package installed
\usepackage{amssymb}  % assumes amsmath package installed
\allowdisplaybreaks % for amsmath package, allow equations to break across pages

\let\labelindent\relax% relax the indent
\usepackage{enumitem}% for enhanced flexibility in customizing the list environments
\usepackage{mathtools}% Mathematical tools to use with amsmath
\usepackage{amsfonts}% fonts from the American Mathematical Society
\usepackage{bm}% math bold fonts
\usepackage{dsfont}
\usepackage{xcolor}
\usepackage{algpseudocode}% for pseudocode
\usepackage{algorithm}% enhanced version for algorithm-like environments
\usepackage{tabularx}

\makeatletter
\newcommand{\multiline}[1]{%
  \begin{tabularx}{\dimexpr\linewidth-\ALG@thistlm}[t]{@{}X@{}}
    #1
  \end{tabularx}
}
\makeatother
\usepackage{multirow}% for multi-row cells in table
\usepackage{caption}% 
\usepackage{cite}

\newtheorem{theorem}{Theorem}
\newtheorem{lemma}{Lemma}
\newtheorem{assumption}{Assumption}

\newtheorem{definition}{Definition}
\newtheorem{proposition}[theorem]{Proposition}
\newtheorem{remark}{Remark}
\DeclareMathOperator*{\argmax}{\arg\!\max}

\newcommand{\mc}{\mathcal}

\newcommand{\mcx}{\mathcal{X}}

\newcommand{\mca}{\mathcal{A}}
\newcommand{\mbr}{\mathbb{R}}

\newcommand{\mbe}{\mathbb{E}}
\newcommand{\mbi}{\mathbb{I}}
\newcommand{\mcp}{\mathcal{P}}
\newcommand{\eps}{\epsilon}
\newcommand{\wto}{\widetilde{O}}

\title{\LARGE \bf
Online Reinforcement Learning in Markov Decision Process\\ Using Linear Programming
}

\author{Vincent Leon and S. Rasoul Etesami% <-this % stops a space
\thanks{Vincent Leon and S. Rasoul Etesami are with Department of Industrial and Enterprise Systems Engineering and Coordinated Science Laboratory, University of Illinois at Urbana-Champaign, Urbana, IL, 61801, USA. 
Email: (leon18, etesami1)@illinois.edu.}% <-this % stops an unwanted space
\thanks{
This material is based upon work supported by the Air Force Office of Scientific Research under award number FA9550-23-1-0107 and the NSF CAREER Award under grant number EPCN-1944403. 
A shorter version of this paper has been presented in the 62nd IEEE Conference on Decision and Control (CDC), Dec. 13-15, 2023, Singapore and published in its proceedings \cite{leon2023online}.}
}

\begin{document}

\maketitle
\thispagestyle{plain}
\pagestyle{plain}

%%%%%%%%%%%%%%%%%%%%%%%%%%%%%%%%%%%%%%%%%%%%%%%%%%%%%%%%%%%%%%%%%%%%%%%%%%%%%%%%
\begin{abstract}
We consider online reinforcement learning in episodic Markov decision process (MDP) with unknown transition function and stochastic rewards drawn from some fixed but unknown distribution. The learner aims to learn the optimal policy and minimize their regret over a finite time horizon through interacting with the environment. We devise a simple and efficient model-based algorithm that achieves $\wto(LX\sqrt{TA})$ regret with high probability, where $L$ is the episode length, $T$ is the number of episodes, and $X$ and $A$ are the cardinalities of the state space and the action space, respectively. The proposed algorithm, which is based on the concept of ``optimism in the face of uncertainty", maintains confidence sets of transition and reward functions and uses occupancy measures to connect the online MDP with linear programming. It achieves a tighter regret bound compared to the existing works that use a similar confidence set framework and improves computational effort compared to those that use a different framework but with a slightly tighter regret bound.
\end{abstract}

{\bf \emph{--Keywords:}
Markov decision process, online learning, reinforcement learning, regret minimization, occupation measure, linear programming.
}
%%%%%%%%%%%%%%%%%%%%%%%%%%%%%%%%%%%%%%%%%%%%%%%%%%%%%%%%%%%%%%%%%%%%%%%%%%%%%%%%

\section{Introduction}\label{sec:intro}

Sequential decision-making in an unknown environment is one of the central problems in many disciplines, such as control and game theory, operations research, and computer science \cite{altman1999constrained,bertsekas2012dynamic,gabel2007successful,vamvoudakis2021handbook}. In that regard, Markov decision process (MDP) provides a powerful paradigm for modeling and analyzing sequential decision-making in the face of uncertainty. In a standard discrete-time MDP, the system is in a certain state at each time, and the decision-maker has to choose an action based on the current state. Depending on the action taken, the decision-maker receives a reward from the environment, and the system transitions to the next state stochastically based on some state transition function. 
In reality, the decision-maker often may not have information about the underlying dynamics of the system and must learn them in real-time by executing different policies while interacting with the environment.
Such problems are referred to as \emph{online reinforcement learning}, and the decision-maker is called the learner or the player. MDPs and reinforcement learning have emerged in many applications such as robotic control \cite{kober2013reinforcement,polydoros2017survey}, game playing \cite{zhang2021multi}, partially observable optimal feedback control \cite{khan2012reinforcement,kiumarsi2018optimal}, cyber-physical security \cite{tran2019safety,etesami2019dynamic}, healthcare  \cite{coronato2020reinforcement}, and among many others \cite{vamvoudakis2021handbook,wang2020reinforcement}. 

In this paper, we study online reinforcement learning in episodic MDPs, where the learner's interaction with the environment proceeds in repeated episodes. In other words, the system ``restarts'' or gets to some initial state after a time period (episode), and learning takes place between episodes. This is a realistic model for many real scenarios. For example, each round of playing chess with an opponent can be viewed as an episode, and each episode contains multiple time steps for moving the chess pieces. Therefore, a chess player learns winning strategies after each round of the game and becomes more expert by means of rounds of play and learning. 

Specifically, we consider an episodic MDP with finite $X$ states, $A$ actions, $T$ episodes, and $L$ steps in each episode. We devise a learning algorithm based on the optimism principle, which comprises two parts: model estimation and policy optimization. The first part estimates the model of MDP, and the second part optimizes the policy for implementation based on the estimated model. Our benchmark to measure the algorithm performance is \emph{regret}, which is defined as the difference between the maximum expected cumulative reward that one can obtain given the MDP model and the expected cumulative reward returned by the algorithm (see Section \ref{sec:prob} for a formal definition). The objective is to design a tractable reinforcement learning algorithm that achieves a regret sublinear in $T$ and polynomial in other parameters.

\subsection{Related Works}\label{subsec:lit}

The concept of ``optimism in the face of uncertainty'' is a generally well-understood and strategic principle for decision-making over a fixed time horizon. It is widely applied in multi-armed bandit and reinforcement learning problems. Under such a principle, the algorithm maintains a plausible estimate of the model and chooses the action that yields the most optimistic reward.
Earlier works such as UCRL algorithm of \cite{auer2006logarithmic}, UCRL2 algorithm of \cite{jaksch2010near-optimal}, and REGAL algorithm of \cite{bartlett2009regularization} apply this principle and study a similar MDP model with unknown environment dynamics where the reward is randomly drawn from some unknown fixed distribution. 
These algorithms are composed of two parts: model estimation and policy optimization. UCRL2 algorithm achieves a regret bound of $\wto(L^{\frac{3}{2}}X\sqrt{TA})$. UCBVI algorithm for episodic MDP, which is introduced in \cite{azar17minimax}, improves this bound to $\wto(L\sqrt{TXA})$ and achieves the best-known regret bound so far.
% In a recent work \cite{azar17minimax} on episodic MDP, the proposed , which achieves the best-known regret bound so far.
UCBVI algorithm adopts a different optimism approach for reinforcement learning, combining the upper confidence bound (UCB) approach with a Q-learning-type algorithm. 
In a more recent work \cite{jin2018q-learning}, the authors propose a variant of model-free Q-learning algorithm that directly parameterizes and updates the value functions without explicitly modeling the environment. The Q-learning algorithm is combined with the UCB algorithm and achieves a regret of $\wto(L^2\sqrt{TXA})$. Another notable algorithm-design principle is posterior sampling, a.k.a. Thompson Sampling. In \cite{agrawal2017optimistic}, the authors present a posterior-sampling-based learning algorithm that achieves a regret bound of $\wto(L\sqrt{TXA})$.

Recently, several works have extended the approach of confidence sets to adversarial MDPs. In adversarial MDP, the reward functions can change arbitrarily between episodes as if they were chosen by an adversary, whereas the environment dynamics still remain fixed. 
\cite{zimin2013online,rosenberg2019onlineconvex,jin2020learning} tackle the reinforcement learning in episodic adversarial MDPs and apply online mirror descent on the space of occupancy measures instead of optimizing policies directly. 
% The O-REPS algorithm in \cite{zimin2013online} achieves a regret bound of $\wto(L\sqrt{T})$ for known transition and full-information feedback and a regret bound of $\wto(\sqrt{TLXA})$ for known transition and bandit feedback. The UC-O-REPS in \cite{rosenberg2019onlineconvex} achieves a regret bound of $\wto(LX\sqrt{TA})$ for unknown transition and full-information feedback. Moreover, the UOB-REPS algorithm in \cite{jin2020learning} achieves a regret bound of $\wto(LX\sqrt{TA})$ for unknown transition and bandit feedback. 
Inspired by UC-O-REPS and UOB-REPS algorithms in \cite{rosenberg2019onlineconvex} and \cite{jin2020learning} respectively, we adopt the concept of occupancy measure in our paper for non-adversarial setting.
% However, different from their work {\color{blue} on adversarial MDP}, instead of applying online mirror descent and solving convex optimization problem, our algorithm solves the optimization problem greedily without regularization using a linear program. 

\subsection{Contributions}

In this work, we devise a reinforcement learning algorithm for episodic MDPs which achieves a regret bound of $\wto(LX\sqrt{TA})$ and requires at most $XA \log T$ times of model update and optimization by combining UCRL2 algorithm and UOB-REPS algorithm and extending their analysis. 
Moreover, the optimization problem for policy update is essentially a linear program and can be solved efficiently using standard linear programming solvers. Our contribution is twofold:
\begin{itemize}
    \item Compared to the best regret bound of existing algorithms that adopt the same method based on confidence set, we reduce the regret bound from $\wto(L^{\frac{3}{2}}X\sqrt{TA})$\footnote{We have translated the bounds into our notations and model defined in Section \ref{sec:prob}. Therefore, the bounds shown here differ slightly from how they appear in the original paper \cite{jaksch2010near-optimal}.} in \cite{jaksch2010near-optimal} to $\wto(LX\sqrt{TA})$. 

    \item Compared to the best regret bound of all existing algorithms regardless of methods or principles, which is $\wto(L\sqrt{TXA})$ in \cite{azar17minimax}, we have improved the computational efficiency by reducing the number of steps for model estimation and policy update from $T$ to at most $XA \log T$ at the cost of an additional factor of $\sqrt{X}$ in the regret bound. 

    % \item Compared to the regret bound of the algorithm in \cite{jin2020learning} that adopts the same confidence set method and deals with a stronger adversarial MDP, we have improved computational complexity by reducing the number of steps for policy update from $T$ to $XA \log T$. Moreover, as opposed to \cite{jin2020learning}, which requires solving an involved convex program at each policy update, we only require solving a simpler linear program.  
\end{itemize}

\noindent
\textbf{Notations:} 
For any non-negative integer $k$, we let $[k] \triangleq \{0, 1, \cdots, k\}$. For any positive integer $k$, we let $[k]_+ \triangleq \{1, \cdots, k\}$. For any finite set $A$, $\Delta(A)$ denotes the probability simplex of $A$. $\mbi\{\cdot\}$ denotes the indicator function. We use $\Omega$ and $O$ to denote the asymptotic lower and upper bounds, respectively, and use $\wto$ to represent $O$ by hiding the logarithmic factors. We use calligraphic and uppercase letters to denote sets and their cardinalities, respectively (e.g., $X = |\mcx|$, $X_h = |\mcx_h|$, $A = |\mca|$).

\section{Problem Formulation}\label{sec:prob}

We study an episodic discrete-time MDP defined by a tuple $\mc{M} = (\mcx, \mca, P, r)$, where $\mcx$ is the finite state space, $\mca$ is the finite action space, $P: \mcx \times \mca \times \mcx \to [0,1]$ is the transition function, and $r: \mcx \times \mca \to [0,1]$ is the expected reward function (shortened as reward function hereinafter). When the learner takes action $a \in \mca$ in state $x \in \mcx$, the system transitions to state $x' \in \mcx$ with probability $P(x'|x,a)$, and the learner receives a realized reward randomly drawn from some distribution on $[0,1]$ with mean $r(x,a)$. The process proceeds in episodes, where each episode has the same length of $L$ time steps. The system restarts at the end of each episode, and a new episode begins with an initial state. Moreover, following the previous literature \cite{neu2012adversarial, zimin2013online, rosenberg2019onlineconvex,rosenberg2019onlinestochastic,jin2020learning}, we shall focus on the \emph{layered episodic MDP} model as described in the following assumption.
\begin{assumption} \label{as:layered-mdp}
The episodic MDP model satisfies the following conditions: 
\begin{itemize}
    \item The state space $\mcx$ can be partitioned into $(L+1)$ non-intersecting layers $\mcx_0, \cdots, \mcx_L$.
    \item $\mcx_0$ and $\mcx_L$ are singletons, i.e., $\mcx_0 = \{x_0\}$, $\mcx_L = \{x_L\}$.
    \item State transition takes place only between two consecutive layers, i.e., $P(x'|x,a)>0$ only if $x \in \mcx_h$ and $x' \in \mcx_{h+1}$ for some $h \in [L-1]$.
    \item The transition function $P$ and the reward function $r$ are time-invariant.
\end{itemize}
\end{assumption}

\begin{remark}
In fact, restriction to the class of layered episodic MDPs is without loss of generality because any general setting of episodic MDP can be reduced to such a layered structure \cite{rosenberg2019onlineconvex, jin2020learning}. The reason is that non-layered structures with time-variant transition and reward functions can be reduced to layered models by creating $L$ identical copies of the state space to construct the layers. Therefore, the algorithm and results in this paper can be applied to any general episodic MDP model up to a scaling factor that depends on the length of each episode.
\end{remark}

%A few examples of reduction are listed as follows. 
%\begin{example}
%\textbf{Initial state randomly drawn from some distribution or decided adversarially:} Such models can be reduced to layered models by creating a dummy initial layer $\{x_0\}$ and setting the transition probability or occupancy measure (to be introduced later) associated with $x_0$ accordingly.
%\end{example}

The learner interacts with the environment over a fixed finite time horizon $T$ and follows a learning algorithm to learn the optimal policy for the MDP. Only the state space $\mcx$ and the action space $\mca$ are known to the learner ahead of time, whereas the transition function $P$ and the reward function $r$ remain unknown. In episode $t \in [T]_+$, the learner determines a stochastic policy $\pi_t: \mcx \to \Delta(\mca)$ where $\pi_t(a|x)$ denotes the probability of choosing $a$ in state $x$ and executes the policy $\pi_t$ over the entire episode. In each step $h \in [L-1]$, the learner chooses an action $a_h \sim \pi_t(\cdot|x_h)$ and receives a realized reward $r_h$ which is randomly drawn from some unknown distribution with mean $r(x_h, a_h)$. Such information feedback is often referred to as the \emph{bandit feedback}.

For any policy $\pi$, we define the expected reward over an episode as follows:
\begin{equation}
    V(\pi) \triangleq \mbe\left[ \sum_{h=0}^{L-1} r(x_h, a_h) \Bigg | P, \pi \right].
\end{equation}
An optimal policy is the policy $\pi^*$ which maximizes the expected reward over an episode, i.e., $\pi^* \in \argmax_{\pi} V(\pi)$. The objective of the learner is to minimize the \emph{total regret}, which is defined as 
\begin{equation} \label{eq:regret}
    \texttt{Reg}_T = T \cdot V(\pi^*) - \sum_{t=1}^T V(\pi_t).
\end{equation}

\section{Preliminaries: Occupancy Measure in Episodic MDP} \label{subsec:prelim}

In this section, we provide some preliminary results and analytical tools that will leverage to establish our main results. We begin with the notion of occupancy measure that was introduced in \cite{altman1999book} and extensively used for online learning in MDPs  \cite{neu2012adversarial,zimin2013online,rosenberg2019onlineconvex,rosenberg2019onlinestochastic,jin2020learning}. 

\begin{definition}
Given a layered episodic MDP, the occupancy measure $q^{P,\pi}: \mcx \times \mca \times \mcx \to [0,1]$ induced by transition function $P$ and policy $\pi$ is defined as 
\begin{equation} \label{eq:def-occ-mea}
    q^{P,\pi}(x, a, x') = \mathbb{P}\{x_{h(x)} = x, a_{h(x)} = a, x_{h(x)+1} = x' | P, \pi\},
\end{equation}
where $h(x)$ is the index of the layer to which state $x$ belongs. 
\end{definition}

In other words, $q^{P,\pi}(x, a, x')$ represents the probability of visiting the state-action-state triple $(x, a, x')$ in the MDP with transition function $P$ under policy $\pi$. Next, let us define $Q$ to be the set of all valid occupancy measures that can be induced by some arbitrary transition function $P$ and some policy $\pi$. We have the following statement about $Q$.

\begin{proposition}[\cite{rosenberg2019onlineconvex,rosenberg2019onlinestochastic,jin2020learning}]
% The set of all feasible occupancy measures 
$Q$ is a non-empty polytope and has the following representation: 
% {\color{blue}
\begin{align}
    Q = & \Bigg\{ q \in \mbr^{XAX}_+ \Big| \notag \\
    & \sum_{x\in \mcx_h} \sum_{a\in \mca} \sum_{x' \in \mcx_{h+1}} q(x, a, x') = 1, \forall h\in [L-1], \notag \\
    & \sum_{x'\in\mcx_{h-1}} \sum_{a\in \mca} q(x', a, x) = \sum_{a\in \mca} \sum_{x'\in\mcx_{h+1}} q(x,a,x'), \notag \\
    & \quad \quad \forall x\in\mcx_h, h \in [L-1]_+,\notag \\
    & q(x,a,x') = 0, \notag \\
    & \quad \quad \forall x\in \mcx_h, a\in \mca, x'\notin \mcx_{h+1}, h\in [L-1] \Bigg\}. \label{eq:plyh-all-occ-mea}
\end{align}
% }
\end{proposition}

Given an episodic MDP with transition function $P$ and policy $\pi$, the induced occupancy measure $q^{P, \pi}$ is uniquely determined by \eqref{eq:def-occ-mea}. Conversely, any occupancy measure $q\in Q$ can induce some transition function $P^q$ and some policy $\pi^q$, as shown in the following proposition. 
\begin{proposition}[\cite{rosenberg2019onlineconvex,rosenberg2019onlinestochastic,jin2020learning}]
Any valid occupancy measure $q\in Q$ can induce a transition function $P^q$ and a policy $\pi^q$ given by
\begin{align}
    &P^q(x'|x,a)  = \frac{q(x,a,x')}{\sum_{y\in \mcx_{h(x)+1}} q(x,a,y)}, \\
    &\pi^q(a|x) = \frac{\sum_{x'\in\mcx_{h(x)+1}} q(x,a,x')}{\sum_{b\in\mc{A}} \sum_{x'\in\mcx_{h(x)+1}} q(x,b,x')}.
\end{align}
Moreover, when the denominators are non-zero, $P^q$ and $\pi^q$ are uniquely defined.
\end{proposition}

Given a fixed transition function $P$, denote by $Q^P$ the set of all valid occupancy measures associated with the transition function $P$. The following proposition shows that $Q^P$ is a polytope, which we shall refer to as the \emph{occupancy measure polytope induced by $P$}.
\begin{proposition}[\cite{altman1999book,rosenberg2019onlineconvex,rosenberg2019onlinestochastic,jin2020learning}]
$Q^P$ is a non-empty polytope and has the following representation:
\begin{align} \label{eq:plyh-P-occ-mea}
    & Q^P = Q \: \cap \notag \\ 
    & \Bigg\{ q \in \mbr^{XAX}_+ \Big|
    q(x,a,x') = P(x'|x,a)\sum_{y\in \mcx_{h+1}}q(x,a,y),  \notag \\
    & \quad \quad \quad \forall x\in \mcx_h, a\in \mca, x'\in \mcx_{h+1}, h\in[L-1] \Bigg\}.
\end{align}
\end{proposition}

The occupancy measure allows us to reduce the task of learning the optimal policy to the task of learning the optimal occupancy measure over $Q^P$. However, $Q^P$ itself is unknown. Therefore, a natural approach is to construct and maintain a set of plausible transition functions $\mcp$ such that with high probability, $P \in \mcp$ and $Q^P \subseteq Q^\mcp$, where $Q^\mcp=\cup_{P\in \mcp} Q^P$ is the set of all occupancy measures associated with the transition functions in $\mcp$, and solve for the optimal occupancy measure over $Q^\mcp$. As learning proceeds, the hope is that the size of $\mcp$ and $Q^\mcp$ shrinks, and eventually $Q^\mcp$ gets as close to $Q^P$ as possible. 

With slight abuse of notation, let us define \(r(x,a,x') = r(x,a)\) for every $(x, a, x') \in \mcx \times \mca \times \mcx$.
It has been shown in \cite{rosenberg2019onlineconvex,rosenberg2019onlinestochastic,jin2020learning} that 
\begin{equation}
    \mbe\left[ \sum_{h=0}^{L-1} r(x_h, a_h) \Bigg | P, \pi \right] = \langle q^{P, \pi}, r \rangle.
\end{equation}
Therefore, if we define $q^* \in \argmax_{q \in Q^P} \langle q, r \rangle$, the regret \eqref{eq:regret} can be represented by 
\begin{equation} \label{eq:regret-occ}
    \texttt{Reg}_T  = T \cdot \langle q^*, r \rangle - \sum_{t=1}^T \langle q^{P, \pi_t}, r \rangle = \sum_{t=1}^T \langle q^* - q^{P, \pi_t}, r \rangle.
\end{equation}

\section{Algorithm}\label{sec:alg}

The main algorithm we propose is called ``Upper Confidence Reinforcement Learning using Linear Programming'' (UCRL-LP) algorithm and is presented in Algorithm \ref{alg:main}. The proposed algorithm is mainly inspired by UCRL \cite{auer2006logarithmic} and UCRL2 \cite{jaksch2010near-optimal} algorithms which use upper confidence bounds to determine an optimistic policy and O-REPS \cite{zimin2013online}, and UC-O-REPS \cite{rosenberg2019onlineconvex} and UOB-REPS \cite{jin2020learning} algorithms which solve online convex optimization on the occupancy measure space. The algorithm consists of two main parts: maintaining confidence sets for transition and reward functions (estimation) and solving a linear program for occupancy measure update (optimization). In the following subsections, we proceed to describe in detail the main steps in Algorithm \ref{alg:main}. 

\begin{algorithm}[t]
\caption{UCRL-LP Algorithm} \label{alg:main}
\begin{algorithmic}
\State \textbf{Input:} State space $\mcx$, action space $\mca$, time horizon $T$, confidence parameter $\delta$.
\State
\State \textbf{Initialization:}
\State Epoch index: $i = 1$. 
\State Confidence sets: 
\vspace{-.75em}
\begin{equation*}
\mcp_1 = \text{set of all transition functions}, Q^{\mcp_1} = Q.
\end{equation*}
\vspace{-2em}
\State Counters, occupancy measure, and estimators $\forall (x,a,x')$:
\vspace{-.5em}
\begin{align*}
    & N_1(x,a) = N_1(x,a,x') = n_1(x,a) = n_1(x,a,x') = 0, \\
    & \hat{q}_1(x,a,x') = \frac{1}{X_h A X_{h+1}}, \\
    & \bar{r}_1(x,a) = 0, \quad \hat{r}_1(x,a) = 1.
\end{align*}
\vspace{-1.5em}
\State Policy: $\pi_1 = \pi^{\hat{q}_1}$.
\State
\For{$t = 1, \cdots, T$}
\State Execute policy $\pi_i$.
\State Obtain trajectory $\{x_h, a_h, r^{(t)}_h\}_{h=0}^{L-1}$.
\State Update counters $\forall h$:
\vspace{-.5em}
\begin{align*}
    n_i(x_h, a_h) & \leftarrow n_i(x_h, a_h) + 1, \\ 
    n_i(x_h, a_h, x_{h+1}) & \leftarrow n_i(x_h, a_h, x_{h+1}) + 1.
\end{align*}
\vspace{-1.5em}
\State Update empirical estimator $\bar{r}_i$ $\forall h$:
\vspace{-.5em}
\begin{equation*}
    \bar{r}_i(x_h, a_h) \leftarrow \bar{r}_i(x_h, a_h) + \frac{r^{(t)}_h - \bar{r}_i(x_h, a_h)}{N_i(x_h, a_h) + n_i(x_h, a_h)}.
\end{equation*}
\vspace{-1em}
\If{$\exists (x, a) \in \mcx \times \mca$, $n_i(x,a) \geq N_i(x,a)$}
\State Let $t(i)$ be the start of epoch $i$. Start a new epoch 
\vspace{-.5em}
\begin{equation*}
    i \leftarrow i+1, \quad t(i) \leftarrow t+1.
\end{equation*}
\vspace{-1.5em}
\State Update counters $\forall (x,a,x')$:
\vspace{-.5em}
\begin{align*}
    N_i(x,a) & \leftarrow N_{i-1}(x,a) + n_{i-1}(x,a), \\
    N_i(x,a,x') & \leftarrow N_{i-1}(x,a,x') + n_{i-1}(x,a,x'), \\
    n_i(x,a) & \leftarrow 0, \quad n_i(x,a,x') \leftarrow 0.
\end{align*}
\vspace{-1.5em}
\State Update empirical estimator $\bar{r}_i$ $\forall (x,a)$: 
\vspace{-.5em}
\begin{equation*}
    \bar{r}_i(x,a) \leftarrow \bar{r}_{i-1}(x,a).
\end{equation*}
\vspace{-1.5em}
\State Update confidence set $\mcp_i$ based on Eq. \eqref{eq:empirical-transition}-\eqref{eq:conf-set-P}.
\State Update $\hat{r}_i$ based on Eq. \eqref{eq:extra-term-r}-\eqref{eq:rew-ucb}.
\State Update occupancy measure $\hat{q}_i$ by solving \ref{eq:lp-occ-update}.
\State Update policy $\pi_i = \pi^{\hat{q}_i}$.
\EndIf
\EndFor

\end{algorithmic}
\end{algorithm}

\subsection{Confidence Sets}

This idea originates from the concept of ``optimism in the face of uncertainty''. The technique of constructing confidence set for transition and reward functions has been widely adopted in the existing literature of online MDP and reinforcement learning \cite{auer2006logarithmic,jaksch2010near-optimal,agrawal2017optimistic,rosenberg2019onlineconvex,rosenberg2019onlinestochastic,jin2020learning}. The algorithm maintains confidence sets for both $P$ and $r$ such that with high probability, $P$ and $r$ lie in those confidence sets, and the estimate becomes more and more accurate as learning proceeds.

To construct such confidence sets, the algorithm keeps track of the number of visits to each $(x,a)$-pair and to each $(x,a,x')$-triplet and computes the empirical mean of transition and reward functions. The confidence set is constructed based on the empirical mean and an extra term that depends on the number of visits. The algorithm UCRL-LP proceeds in epochs of random length, and a doubling epoch schedule is adopted. The confidence sets for $P$ and $r$ and the occupancy measure are updated only at the beginning of each epoch. The first epoch starts at $t = 1$, and each epoch will end whenever the number of visits to some $(x,a)$-pair is doubled compared to its value at the beginning of the epoch. Such a technique is deployed in existing algorithms such as UCRL2 \cite{jaksch2010near-optimal}, UC-O-REPS \cite{rosenberg2019onlineconvex} and UOB-REPS \cite{jin2020learning}.
This ensures that the confidence sets and the occupancy measure to be implemented will be updated only when there are significant improvements in the confidence sets. Therefore, this technique largely improves the computational efficiency of the algorithm without negatively affecting its performance. The total number of steps for policy (occupancy measure) update is at most $XA \log T$ \cite{jaksch2010near-optimal}, compared to $T$ steps of policy update from general reinforcement learning algorithms which update the policy after each episode.

Specifically, let $N_i(x, a)$ and $N_i(x, a, x')$ be the number of visits to $(x, a)$-pair and $(x, a, x')$-triplet at the beginning of epoch $i$. The empirical estimation of $P$ in epoch $i$ is 
\begin{equation} \label{eq:empirical-transition}
    \bar{P}_i(x'|x, a) = \frac{N_i(x,a,x')}{\max\{1, N_i(x,a)\}}.
\end{equation}
The confidence set $\mcp_i$ includes the transition functions that are ``close enough'' to $\bar{P}$ with respect to a confidence radius $\eps_i(x,a)$ defined by 
\begin{equation}\label{eq:extra-term-P}
    \eps_i(x,a) = \sqrt{\frac{2X_{h(x)+1} \log (TXA/\delta)}{\max\{1, N_i(x,a)\}}},
\end{equation}
where $\delta \in (0,1)$ is some confidence parameter.
The confidence set $\mcp_i$ is defined by 
\begin{align}
    \mcp_i = \{P: \lVert P(\cdot|x,a) - \bar{P}_i(\cdot|x,a) \rVert_1 \leq \eps_i(x,a), & \notag \\
    \forall x\in \mcx, a\in \mca\}. & \label{eq:conf-set-P}
\end{align}
The following lemma provides a closed-form description of $Q^{\mcp_i}$, which can be represented by a system of linear constraints. This allows us to update the occupancy measure efficiently by simply solving a linear program over the polytope $Q^{\mcp_i}$. 
\begin{lemma}[Theorem 4.2 in \cite{rosenberg2019onlineconvex}] \label{lem:conf-set-polytope}
If $\mcp_i$ is defined by Eq. \eqref{eq:conf-set-P}, $Q^{\mcp_i}$ is a non-empty polytope and has the following representation:
% {\color{blue}
\begin{align}
    & Q^{\mcp_i} = Q \: \cap \: \Bigg\{ q \in \mbr^{XAX}_+, \eps \in \mbr^{XAX}_+ \Big| \notag \\
    & q(x,a,x') - \bar{P}_i(x'|x,a) \sum_{y\in \mcx_{h+1}} q(x,a,y) \leq \eps(x,a,x'),  \notag \\
    & \quad \quad \quad \forall x\in \mcx_h, a\in \mca, x'\in \mcx_{h+1}, h\in[L-1], \notag \\
    & \bar{P}_i(x'|x,a) \sum_{y\in \mcx_{h+1}} q(x,a,y) - q(x,a,x') \leq \eps(x,a,x'),  \notag \\
    & \quad \quad \quad \forall x\in \mcx_h, a\in \mca, x'\in \mcx_{h+1}, h\in[L-1], \notag \\
    & \sum_{x'\in \mcx_{h+1}} \eps(x,a,x') \leq \eps_i(x,a) \sum_{x'\in \mcx_{h+1}} q(x,a,x'), \notag \\
    & \qquad \qquad \qquad \qquad \forall x\in \mcx_h, a\in \mca, h\in [L-1] \Bigg\}.
\end{align}
% }
\end{lemma}

Next, we consider the following lemma, which states that $\mcp_i$ contains $P$ with high probability. 
\begin{lemma}[Lemma 4.1 in \cite{rosenberg2019onlineconvex}] \label{lem:conf-set-P-high-prob}
For any $\delta \in (0,1)$, with probability at least $1 - \delta$, $P \in \mcp_i$ for all $i$.
\end{lemma}

Similarly, the empirical estimation of $r$ in epoch $i$ is 
\begin{equation}
    \bar{r}_i(x,a) = \frac{\sum_{t=1}^{t(i)-1} \mbi^{(t)} \{(x,a)\} \cdot r^{(t)}_h}{\max\{1, N_i(x,a)\}},
\end{equation}
where $t(i)$ denotes the starting time of epoch $i$, $\mbi^{(t)}\{(x,a)\}$ is the indicator function that equals 1 if the $(x,a)$-pair is visited in round $t$ and 0 otherwise, and $r_h^{(t)}$ denotes the realized reward at step $h$ in round $t$. The confidence radius is defined by
\begin{equation} \label{eq:extra-term-r}
    b_i(x,a) = \sqrt{\frac{2 \log (TXA/\delta)}{\max\{1, N_i(x,a)\}}}.
\end{equation}
The upper and lower confidence bounds (UCB and LCB, respectively) for $r$ are defined by 
\begin{align}
    \hat{r}_i(x,a) & = \min \{1, \bar{r}_i(x,a) + b_i(x,a)\}, \label{eq:rew-ucb}\\
    \check{r}_i(x,a) & = \max \{0, \bar{r}_i(x,a) - b_i(x,a)\}. \label{eq:rew-lcb}
\end{align}
A similar lemma below states that $r(x,a)$ is bounded by $\hat{r}_i(x,a)$ and $\check{r}_i(x,a)$ with high probability. The proof of this lemma follows directly using Hoeffding inequality and union bound.
\begin{lemma} \label{lem:conf-set-rew}
    For any $\delta \in (0,1)$, with probability at least $1 - 2\delta$, $\check{r}_i(x,a) \leq r(x,a) \leq \hat{r}_i(x,a)$ for all $i$ and all $(x,a) \in \mcx \times \mca$.
\end{lemma}

\subsection{Linear Programming for Occupancy Measure Update}

In each epoch, the algorithm updates the occupancy measure by solving a linear program. As we have seen in Section \ref{sec:prob}, the reinforcement learning problem is associated with an optimization problem over some occupancy measure space. As a result of Lemma \ref{lem:conf-set-polytope}, such an optimization problem reduces to a linear program, which can be efficiently solved by standard linear program solvers. Moreover, the occupancy measure is updated only at the beginning of each epoch, requiring at most $XA \log T$ steps for the update. 
% This is different from the adversarial setting considered in \cite{rosenberg2019onlineconvex,rosenberg2019onlinestochastic,jin2020learning}, where the occupancy measure is updated in each round and needs $T$ steps of update, even though the number of steps of confidence set update is bounded by $XA \log T$.

Finally, to update occupancy measure in epoch $i$, the algorithm determines an optimistic occupancy measure by solving the following linear program: 
\begin{equation} \label{eq:lp-occ-update}
    \hat{q}_i \in \argmax_{q \in Q^{\mcp_i}} \: \langle q, \hat{r}_i \rangle. \tag{LP1}
\end{equation}
The algorithm implements the policy $\pi_i \triangleq \pi^{\hat{q}_i}$ induced by $\hat{q}_i$ for entire epoch $i$.

\section{Main Results}\label{sec:main}

In this section, we present the regret bound of UCRL-LP algorithm. The proof follows by decomposing the regret into four components and bounding each of those terms separately. 

\begin{theorem} \label{thm:main}
With probability at least $1 - 5\delta$, UCRL-LP algorithm achieves the following regret: 
\begin{equation}
    \texttt{Reg}_T \leq O \left( L X \sqrt{T A \log \left( \frac{TXA}{\delta} \right)} \right).
\end{equation}
\end{theorem}
\begin{proof}
Define $q_i \triangleq q^{P, \pi_i}$, i.e., $q_i$ is the occupancy measure induced by the policy $\pi_i$ implemented in epoch $i$ and the true (unknown) transition function $P$. Note that $\hat{q}_i$ and $q_i$ are not the same: they are induced by the same policy $\pi_i$; however, they do not necessarily induce the same transition function. Furthermore, denote by $i(t)$ the index of epoch for round $t \in [T]_+$.
Following \cite{jin2020learning}, the regret defined by Eq. \eqref{eq:regret-occ} can be decomposed into the following four terms: 
\begin{align}\label{thm:final-decomposition}
    \texttt{Reg}_T & = \sum_{t=1}^T \langle q^* - q_{i(t)}, r \rangle \notag \\
    & = \underbrace{\sum_{t=1}^T \langle q^*, r - \hat{r}_{i(t)} \rangle}_{\textsc{Bias-I}} + \underbrace{\sum_{t=1}^T \langle q^* - \hat{q}_{i(t)}, \hat{r}_{i(t)} \rangle}_{\textsc{Diff}} \notag \\
    & \quad + \underbrace{\sum_{t=1}^T \langle \hat{q}_{i(t)} - q_{i(t)}, \hat{r}_{i(t)} \rangle}_{\textsc{Error}} + \underbrace{\sum_{t=1}^T \langle q_{i(t)}, \hat{r}_{i(t)} - r \rangle}_{\textsc{Bias-II}}.
\end{align}
\textsc{Bias-I} and \textsc{Bias-II} are due to the use of a biased estimator $\hat{r}_i$ in the algorithm. \textsc{Diff} represents the difference in solutions to the linear program. \textsc{Error} measures the error in estimating $P$. Each term will be analyzed separately.

Firstly, as a result of Lemma \ref{lem:conf-set-rew}, with probability at least $1 - 2\delta$, $\hat{r}_{i(t)}(x,a)$ is an upper bound of $r(x,a)$ for all $t\in [T]_+$ and all $(x,a) \in \mcx \times \mca$. Therefore, $\textsc{Bias-I} \leq 0$ with probability at least $1 - 2\delta$.
Then, as a result of Lemma \ref{lem:conf-set-P-high-prob}, with probability at least $1-\delta$, $Q^P \subseteq Q^{\mcp_{i(t)}}$, and hence, $q^* \in Q^{\mcp_{i(t)}}$ for all $t\in [T]_+$. By the definition of $\hat{q}_{i(t)}$ in \ref{eq:lp-occ-update} and by optimality, with probability at least $1-\delta$, $\textsc{Diff} \leq 0$. Next, we proceed to bound the remaining terms \textsc{Error} and \textsc{Bias-II}.

\noindent
{\bf Bounding \textsc{Error}:} The term \textsc{Error} is due to the learner's lack of knowledge about transition function $P$ and hence comes from the estimation error for $P$, which is associated with the confidence set $\mcp_i$. The following lemma is the key lemma for bounding \textsc{Error}, the proof of which is in Appendix \ref{ap:proof-error}.

\begin{lemma} \label{lem:key-lem}
With probability at least $1 - 2\delta$, UCRL-LP algorithm ensures that 
\begin{equation}
    \sum_{t=1}^T \lVert \hat{q}_{i(t)} - q_{i(t)} \rVert_1 \leq 12 LX \sqrt{TA \log \left( \frac{TXA}{\delta} \right)}.
\end{equation}
\end{lemma}

Lemma \ref{lem:key-lem} bounds the accumulated difference between occupancy measures induced by the implemented policy over the actual and the estimated transition functions. As a result, \textsc{Error} can be bounded using the following lemma, whose proof uses Lemma \ref{lem:key-lem} and is given in Appendix I. 

\begin{lemma}\label{lem:error}
With probability at least $1 - 2\delta$, 
\begin{equation}
    \textsc{Error} \leq O \left( LX \sqrt{TA \log \left( \frac{TXA}{\delta} \right)} \right).
\end{equation}
\end{lemma}

\noindent 
{\bf Bounding \textsc{Bias-II}:} The term \textsc{Bias-II} represents the accumulated difference between the UCB of $r$ and its expectation. Note that Lemma \ref{lem:conf-set-rew} states that a high-confidence lower bound of $r$ is $\check{r}_{i(t)}$ for all $t\in [T]_+$. Therefore, it suffices to bound $\sum_{t=1}^T \langle q_{i(t)}, \hat{r}_{i(t)} - \check{r}_{i(t)} \rangle$. \textsc{Bias-II} is bounded by the following lemma, the proof of which can be found in Appendix \ref{ap:proof-bias}.

\begin{lemma}\label{lem:bias}
With probability at least $1 - 2\delta$,
\begin{equation}
    \textsc{Bias-II} \leq O \left(\sqrt{TLXA \log \left(\frac{TXA}{\delta}\right)}\right).
\end{equation}
\end{lemma}

Finally, combining Lemmas \ref{lem:error} and \ref{lem:bias} with \eqref{thm:final-decomposition} and using the fact that $\textsc{Bias-I} \leq 0$ and $\textsc{Diff} \leq 0$ completes the proof.    
\end{proof}

% {\color{blue}
\section{Simulation Results}

In this section, we show the simulation results to compare the performance of our UCRL-LP algorithm with the performance of UCRL2 algorithm. The simulation results are presented in Fig. \ref{fig:sim}.
Both algorithms are tested on a layered MDP with 7 layers ($L=6$), 10 states in each of the intermediate layers ($X = 52$), and 5 actions ($A = 5$). In each trial, a new MDP model with such a structure is randomly generated, and two algorithms are tested on this MDP model, respectively. For each value of $T$, 10 trials have been carried out. The solid lines in Fig. \ref{fig:sim} represent the average regrets over 10 trials, and the filled areas represent the range of regrets of 10 trials, blue color for UCRL-LP algorithm and orange color for UCRL2 algorithm, respectively. 
Fig. \ref{fig:sim} shows that both algorithms achieve sublinear regret in terms of $T$. Moreover, our UCRL-LP algorithm achieves significantly better regret than UCRL2 algorithm. 

\begin{figure}[h]
    \centering
    \includegraphics[width=.48\textwidth]{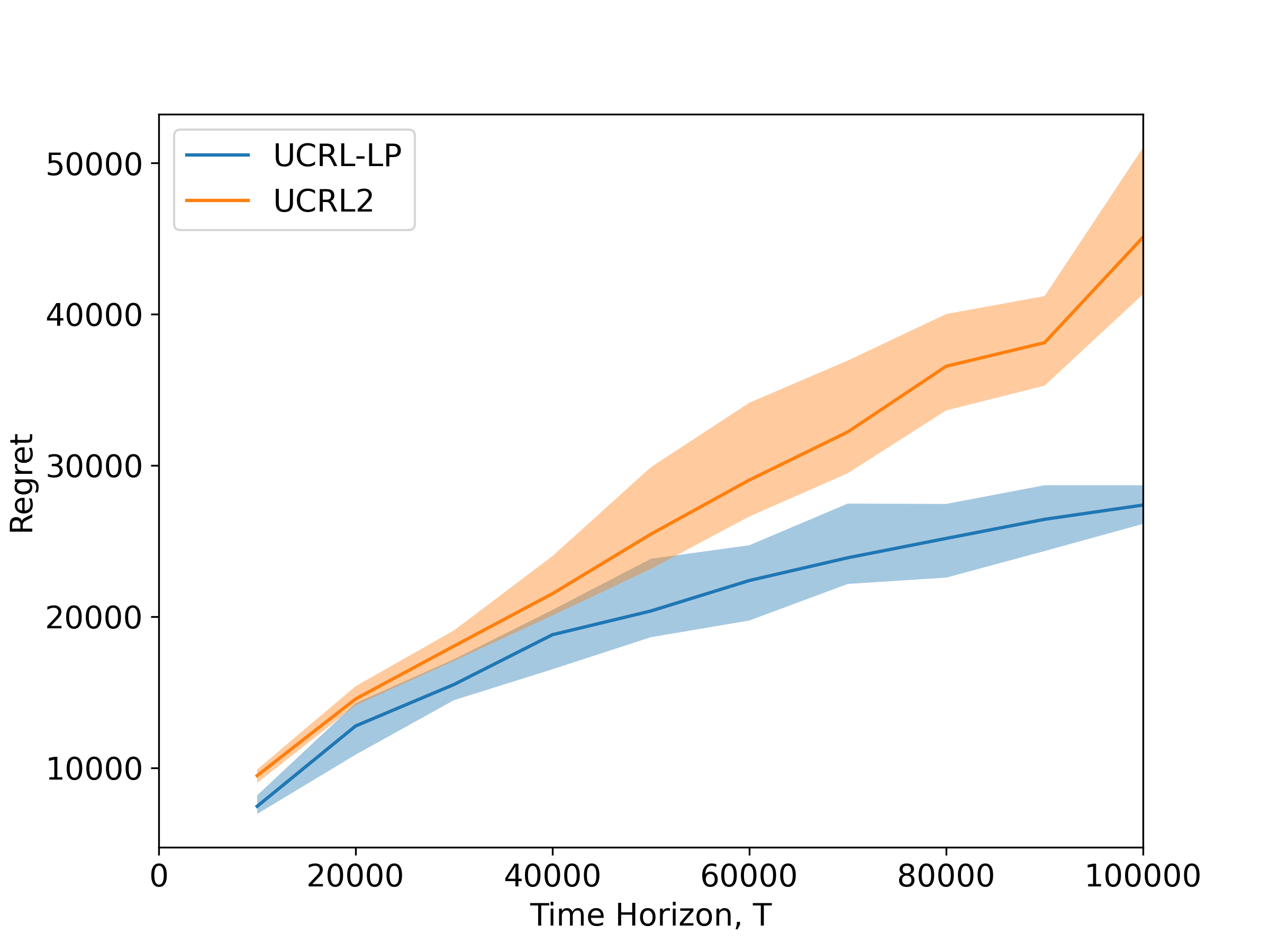}
    \caption{Performance of UCRL-LP and UCRL2 algorithms.}
    \label{fig:sim}
\end{figure}

% }

\section{Conclusion}\label{sec:concl}

In this paper, we devise an efficient algorithm for reinforcement learning in episodic MDP which achieves $\wto(LX \sqrt{TA})$ regret and requires at most $XA \log T$ times of solving a linear program for its policy updates. The algorithm maintains confidence sets and solves linear optimization on polytope of occupancy measure which can be accomplished by standard linear programming solvers. Our contributions include: i) reducing an extra factor of $\sqrt{L}$ in the regret bound compared to the earlier works that deploy confidence set framework, ii) reducing the computational complexity from $O(T)$ to $O(\log T)$ compared to the earlier works that adopt a different method but with a slightly tighter regret bound.
% and iii) reducing the number of optimization oracles that need to be solved during the policy update process from $O(T)$ to $O(\log T)$ and using only linear programming solvers. 
An interesting future research direction is removing a $\sqrt{X}$ factor in the regret bound by combining the occupancy measure framework used in this work with some Q-learning-based algorithm.

\bibliographystyle{IEEEtran}
\bibliography{references}

\appendices

\section{Proof of Lemma \ref{lem:key-lem} and \ref{lem:error}} \label{ap:proof-error}

We follow the main steps in \cite{rosenberg2019onlineconvex} and \cite{jin2020learning} and use additional ideas to prove Lemmas \ref{lem:key-lem} and \ref{lem:error}. Define $P_i \triangleq P^{\hat{q}_i}$, the transition function induced by $\hat{q}_i$ (see Eq. \eqref{eq:lp-occ-update} for the definition of $\hat{q}_i$) in epoch $i$. Define $\xi_i(x,a) \triangleq \lVert P_i(\cdot | x, a) - P(\cdot | x, a) \rVert_1$ for epoch $i$. With slight abuse of notation, define $q(x,a) \triangleq \sum_{x'\in \mcx_{h(x)+1}} q(x,a,x')$ for all $(x,a) \in \mcx \times \mca$ and all occupancy measure $q \in Q$.

\begin{lemma}[Lemma B.1 in \cite{rosenberg2019onlineconvex}]\label{lem:occ-1-norm}
\begin{align}
    & \sum_{t=1}^T \lVert \hat{q}_{i(t)} - q_{i(t)} \rVert_1 \notag \\
    & \quad \leq \sum_{t=1}^T \sum_{h=0}^{L-1} \sum_{x_h\in \mcx_h} \sum_{a_h\in \mca} \lvert \hat{q}_{i(t)}(x_h,a_h) - q_{i(t)}(x_h,a_h) \rvert \notag \\
    & \quad \quad + \sum_{t=1}^T \sum_{h=0}^{L-1} \sum_{x_h\in \mcx_h} \sum_{a_h\in \mca} q_{i(t)}(x_h,a_h) \xi_{i(t)}(x_h,a_h).
\end{align}
\end{lemma}

\begin{lemma}[Lemma B.2 in \cite{rosenberg2019onlineconvex}]\label{lem:occ-term-1}
For all $t \in [T]_+$ and all $h \in [L-1]_+$, it holds that 
\begin{align}
    & \sum_{x_h\in \mcx_h} \sum_{a_h\in \mca} \lvert \hat{q}_{i(t)}(x_h,a_h) - q_{i(t)}(x_h,a_h) \rvert \notag \\
    & \quad \leq \sum_{s=0}^{h-1} \sum_{x_s \in \mcx_s} \sum_{a_s \in \mca} q_{i(t)}(x_s,a_s) \xi_{i(t)}(x_s,a_s).
\end{align}
\end{lemma}

Combining Lemma \ref{lem:occ-1-norm} and Lemma \ref{lem:occ-term-1}, we get 
\begin{align} \label{eq:occ-1-norm-comb}
    & \sum_{t=1}^T \lVert \hat{q}_{i(t)} - q_{i(t)} \rVert_1 \\
    & \quad \leq \sum_{t=1}^T \sum_{h=1}^{L} \sum_{s=0}^{h-1} \sum_{x_s\in \mcx_s} \sum_{a_s\in \mca} q_{i(t)}(x_s,a_s) \xi_{i(t)}(x_s,a_s). \notag 
\end{align}

By Lemma \ref{lem:conf-set-P-high-prob}, with probability at least $1-\delta$, for all $t\in [T]_+$, 
\begin{align}
    \xi_{i(t)}(x,a) & = \lVert P_{i(t)}(\cdot | x, a) - P(\cdot | x, a) \rVert_1 \notag \\
    & \leq \underbrace{\lVert P_{i(t)}(\cdot | x, a) - \bar{P}_{i(t)}(\cdot | x, a) \rVert_1}_{\leq \eps_{i(t)}(x,a) \text{ because } P_{i(t)} \in \mcp_{i(t)}} \notag \\
    & \quad + \underbrace{\lVert P(\cdot | x, a) - \bar{P}_{i(t)}(\cdot | x, a) \rVert_1}_{\leq \eps_{i(t)}(x,a) \text{ w.p. } \geq 1-\delta \text{ by Lemma \ref{lem:conf-set-P-high-prob}}} \notag \\
    & \leq 2 \eps_{i(t)}(x,a).
\end{align}

The following lemma bounds the right-hand side of Eq. \eqref{eq:occ-1-norm-comb}.

\begin{lemma}[Lemma B.3 in \cite{rosenberg2019onlineconvex}] \label{lem:occ-term-2}
With probability at least $1 - 2\delta$, it holds that 
\begin{align}
    & \sum_{t=1}^T \sum_{h=1}^{L} \sum_{s=0}^{h-1} \sum_{x_s\in \mcx_s} \sum_{a_s\in \mca} q_{i(t)}(x_s,a_s) \xi_{i(t)}(x_s,a_s) \notag \\
    & \quad \leq 6 LX \sqrt{2TA \log \left(\frac{TXA}{\delta}\right)} + 2 LX \sqrt{2T \log \left(\frac{L}{\delta}\right)} \notag \\
    & \quad \leq 12 LX \sqrt{TA \log \left( \frac{TXA}{\delta} \right)}.
\end{align}
\end{lemma}

This completes the proof of Lemma \ref{lem:key-lem}. To prove Lemma \ref{lem:error}, it is easy to see that
\begin{align}
    \textsc{Error} & = \sum_{t=1}^T \langle \hat{q}_{i(t)} - q_{i(t)}, \hat{r}_{i(t)} \rangle \notag \\
    & \leq \sum_{t=1}^T \lVert \hat{q}_{i(t)} - q_{i(t)} \rVert_1 \quad \text{(because $\hat{r}_{i(t)}(x,a) \leq 1$)} \notag \\
    & \leq 12 LX \sqrt{TA \log \left( \frac{TXA}{\delta} \right)},
\end{align}
which completes the proof of Lemma \ref{lem:error}.

\section{Proof of Lemma \ref{lem:bias}} \label{ap:proof-bias}

The following concentration inequality is used in the proof. 

\begin{lemma} [Azuma-Hoeffding ineq., Lemma A.7 in \cite{cesabianchi2006book}] \label{lem:azuma-hoeffding}
Let \(Y_1, Y_2, \cdots\) be a martingale difference sequence and \(\lvert Y_i \rvert \leq c_i\) almost surely for some \(c_i > 0\). For any \(\eps > 0\), 
\begin{equation} \label{eq:azuma-1}
    P\left\{\sum_{i=1}^n Y_i > \eps \right\} \leq \exp \left(-\frac{\eps^2}{2\sum_{i=1}^n c_i^2}\right),
\end{equation}
and 
\begin{equation} \label{eq:azuma-2}
    P\left\{\sum_{i=1}^n Y_i < -\eps \right\} \leq \exp \left(-\frac{\eps^2}{2\sum_{i=1}^n c_i^2}\right).
\end{equation}
\end{lemma}

By Lemma \ref{lem:conf-set-rew}, with probability at least $1 - \delta$, for all $t \in [T]_+$, $\hat{r}_{i(t)}(x,a) - r(x,a) \leq 2b_{i(t)}(x,a)$.
Same as before, let $q(x,a) = \sum_{x'\in \mcx_{h(x)+1}} q(x,a,x')$ for all $(x,a) \in \mcx \times \mca$ and all occupancy measure $q \in Q$.
Therefore, with probablity at least $1 - \delta$, 
\begin{align}
    & \textsc{Bias-II} \notag \\
    = & \sum_{t=1}^T \langle q_{i(t)}, \hat{r}_{i(t)} - r \rangle \notag \\
    \leq & 2 \sum_{t=1}^T \sum_{h=0}^{L-1} \sum_{x_h\in \mcx_h} \sum_{a_h\in \mca} q_{i(t)}(x_h,a_h) b_{i(t)}(x_h,a_h) \notag \\
    = & 2 \sum_{t=1}^T \sum_{h=0}^{L-1} \sum_{x_h\in \mcx_h} \sum_{a_h\in \mca} \notag \\
    & \qquad \qquad (q_{i(t)}(x_h,a_h) - \mbi^{(t)}\{(x_h, a_h)\}) b_{i(t)}(x_h,a_h) \notag \\
    & + 2 \sum_{t=1}^T \sum_{h=0}^{L-1} \sum_{x_h\in \mcx_h} \sum_{a_h\in \mca} \mbi^{(t)}\{(x_h, a_h)\} \cdot b_{i(t)}(x_h,a_h). \label{eq:bias-decomp}
\end{align}
    
We will bound the two terms separately. Lemma \ref{lem:bias-term-1} below bounds the first term.

\begin{lemma} \label{lem:bias-term-1}
With probability at least $1 - \delta$,  
\begin{align}
    & \sum_{t=1}^T \sum_{h=0}^{L-1} \sum_{x_h\in \mcx_h} \sum_{a_h\in \mca} \notag \\
    & \qquad (q_{i(t)}(x_h,a_h) - \mbi^{(t)}\{(x_h, a_h)\}) b_{i(t)}(x_h,a_h) \notag \\
    & \qquad \qquad \qquad \qquad \qquad \leq 2L \sqrt{T} \log \left(\frac{TXA}{\delta}\right).
\end{align}
\end{lemma}

\begin{proof}
Fix some layer $h \in [L-1]$. Consider the following term: 
\begin{align*}
    & \sum_{x_h\in \mcx_h} \sum_{a_h\in \mca} (q_{i(t)}(x_h,a_h) - \mbi^{(t)}\{(x_h, a_h)\}) b_{i(t)}(x_h,a_h) \\
    & = \sum_{x_h\in \mcx_h} \sum_{a_h\in \mca} \frac{q_{i(t)}(x_h,a_h) - \mbi^{(t)}\{(x_h, a_h)\}}{\sqrt{\max\{1, N_{i(t)}(x_h,a_h)\}}} \notag \\ 
    & \qquad \qquad \qquad \qquad \qquad \qquad \qquad \cdot \sqrt{2 \log \left( \frac{TXA}{\delta}\right)}.
\end{align*}
Let $Y_t = \sum_{x_h\in \mcx_h} \sum_{a_h\in \mca} \frac{q_{i(t)}(x_h,a_h) - \mbi^{(t)}\{(x_h, a_h)\}}{\sqrt{\max\{1, N_{i(t)}(x_h,a_h)\}}}$. Hence, $Y_1, \cdots, Y_T$ is a martingale difference sequence. Moreover, $-1 \leq Y_t \leq 1$ and $|Y_t| \leq 1$ for all $t\in [T]_+$. We apply Azuma-Hoeffding inequality (Lemma \ref{lem:azuma-hoeffding}) and will get  
\begin{align}
    & \sum_{t=1}^T \sum_{x_h\in \mcx_h} \sum_{a_h\in \mca} \frac{q_{i(t)}(x_h,a_h) - \mbi^{(t)}\{(x_h, a_h)\}}{\sqrt{\max\{1, N_{i(t)}(x_h,a_h)\}}} \notag \\
    & \qquad \qquad \qquad \qquad \qquad \qquad \leq \sqrt{2 T \log \left(\frac{L}{\delta}\right)} 
\end{align}
with probability at least $1 - \delta/L$.
Taking a union bound for all $h\in [L-1]$ completes the proof. 
\end{proof}

Lemma \ref{lem:bias-term-2} below bounds the second term in Eq. \eqref{eq:bias-decomp}.

\begin{lemma} \label{lem:bias-term-2}
It holds that 
\begin{align}
    & \sum_{t=1}^T \sum_{h=0}^{L-1} \sum_{x_h\in \mcx_h} \sum_{a_h\in \mca} \mbi^{(t)}\{(x_h, a_h)\} \cdot b_{i(t)}(x_h,a_h) \notag \\
    & \qquad \qquad \qquad \qquad \leq 5 \sqrt{TLXA \log \left(\frac{TXA}{\delta}\right)}.
\end{align}
\end{lemma}

\begin{proof}
Firstly, let $m$ denote the number of epochs. It has been shown in \cite{jaksch2010near-optimal} that 
\begin{equation*}
    \sum_{i=1}^m \frac{n_i(x,a)}{\sqrt{N_i(x,a)}} \leq (\sqrt{2}+1) \sqrt{N_m(x,a)} \leq 3 \sqrt{N_m(x,a)},
\end{equation*}
where $n_i(x,a)$ is the number of visits to $(x,a)$-pair in epoch $i$.
By Jensen's inequality we have 
\begin{align*}
    & \sum_{x_h\in \mcx_h} \sum_{a_h\in \mca} \sum_{i=1}^m \frac{n_i(x_h,a_h)}{\sqrt{N_i(x_h,a_h)}} \\
    \leq & 3 \sum_{x_h\in \mcx_h} \sum_{a_h\in \mca} \sqrt{N_m(x_h,a_h)} \\
    \leq & 3 \sqrt{X_h A \sum_{x_h\in \mcx_h} \sum_{a_h\in \mca} N_m(x_h,a_h)} \\
    = & 3 \sqrt{T X_h A}.
\end{align*}
Therefore, 
\begin{align}
    & \sum_{t=1}^T \sum_{h=0}^{L-1} \sum_{x_h\in \mcx_h} \sum_{a_h\in \mca} \mbi^{(t)}\{(x_h, a_h)\} \cdot b_{i(t)}(x_h,a_h) \notag \\
    = & \sum_{t=1}^T \sum_{h=0}^{L-1} \sum_{x_h\in \mcx_h} \sum_{a_h\in \mca} \frac{\mbi^{(t)}\{(x_h, a_h)\}}{\sqrt{\max\{1, N_{i(t)}(x_h,a_h)\}}} \notag \\
    & \qquad \qquad \qquad \qquad \qquad \qquad \qquad \cdot \sqrt{2 \log \left(\frac{TXA}{\delta}\right)} \notag \\
    = & \sum_{h=0}^{L-1} \sum_{x_h\in \mcx_h} \sum_{a_h\in \mca} \sum_{i=1}^m \frac{n_i(x_h,a_h)}{\sqrt{\max\{1, N_i(x_h,a_h)\}}} \notag \\
    & \qquad \qquad \qquad \qquad \qquad \qquad \qquad \cdot \sqrt{2 \log \left(\frac{TXA}{\delta}\right)} \notag \\
    \leq & 3 \sqrt{2 \log \left(\frac{TXA}{\delta}\right)} \sum_{h=0}^{L-1} \sqrt{T X_h A} \notag \\
    \leq & 3 \sqrt{2 TLXA \log \left(\frac{TXA}{\delta}\right)} \quad \text{(by Jensen's inequality)} \notag \\
    \leq & 5 \sqrt{TLXA \log \left(\frac{TXA}{\delta}\right)}.
\end{align}
\end{proof}

Combining Lemma \ref{lem:bias-term-1} and Lemma \ref{lem:bias-term-2} finishes the proof.

% that's all folks
\end{document}